\documentclass[english,final]{llncs}

\usepackage[utf8]{inputenc}
\usepackage[T1]{fontenc}
\usepackage[english]{babel}
\usepackage{graphicx}

\usepackage{textcomp}

\usepackage{setspace}
\usepackage{xcolor}

\usepackage{hyperref}
\hypersetup{
  a4paper,
  bookmarksnumbered,
  bookmarksopen,
  bookmarksopenlevel=0,
  colorlinks=false,
    citebordercolor=teal,
    linkbordercolor=pink,
	pdfauthor={Evgenij Thorstensen},
	pdftitle={Lifting structural tractability to CSP with global constraints},
	pdfsubject={Constraint Programming},
	pdfkeywords={CS,constraints,databases,tractability},
   urlcolor=black,
   implicit=false,
}
\usepackage[sort]{cleveref}

\usepackage{verbatim}

\usepackage{complexity}

\usepackage{mathtools}
\usepackage{amsfonts}

\usepackage{algorithm}
\usepackage{algpseudocode}

\usepackage{cite}

\title{Lifting Structural Tractability to CSP with Global Constraints}

\author{
Evgenij Thorstensen\thanks{Work supported by EPSRC grant EP/G055114/1}
}

\institute{
Department of Computer Science, University of Oxford, UK\\ 
\email{evgenij.thorstensen@cs.ox.ac.uk}
}

\newcommand{\tw}{\ensuremath{\mathsf{tw}}}
\newcommand{\hw}{\ensuremath{\mathsf{hw}}}

\newcommand{\fhw}{\ensuremath{\mathsf{fhw}}}
\newcommand{\subw}{\ensuremath{\mathsf{subw}}}

\newcommand{\vars}{\ensuremath{\mathcal{V}}}
\newcommand{\sol}{\ensuremath{\mathsf{sol}}}

\newcommand{\CSP}{\ensuremath{\textup{CSP}}}

\newcommand{\hyp}{\ensuremath{\mathsf{hyp}}}

\newcommand{\iv}{\ensuremath{\mathsf{iv}}}

\newcommand{\ic}{\ensuremath{\mathsf{ic}}}

\newcommand{\Pj}{\ensuremath{\mathsf{pj}}}

\newcommand{\Extlang}{\ensuremath{\mathbf{Ext}}}
\DeclarePairedDelimiter{\tup}{\langle}{\rangle}

\newcommand{\languageword}{catalogue}
\newcommand{\languagesymbol}{\ensuremath{\Gamma}}

\crefname{chapter}{Chapter}{Chapters}
\crefname{section}{Section}{Sections}
\crefname{subsection}{Section}{Sections}

\crefname{definition}{Definition}{Definitions}
\crefname{theorem}{Theorem}{Theorems}
\crefname{lemma}{Lemma}{Lemmata}
\crefname{propositions}{Proposition}{Propositions}
\crefname{corollary}{Corollary}{Corollaries}
\crefname{property}{Property}{Properties}
\crefname{conjecture}{Conjecture}{Conjectures}

\crefname{example}{Example}{Examples}
\crefname{problem}{Problem}{Problems}
\crefname{notation}{Notation}{Notation}
\crefname{algo}{Algorithm}{Algorithms}
\crefname{algorithm}{Algorithm}{Algorithms}
\crefname{figure}{Figure}{Figures}

\begin{document}

\maketitle

\begin{abstract}
  A wide range of problems can be modelled as constraint satisfaction
  problems (CSPs), that is, a set of constraints that must be
  satisfied simultaneously. Constraints can either be represented
  extensionally, by explicitly listing allowed combinations of values,
  or implicitly, by special-purpose algorithms provided by a
  solver. Such implicitly represented constraints, known as global
  constraints, are widely used; indeed, they are one of the key
  reasons for the success of constraint programming in solving
  real-world problems.

  In recent years, a variety of restrictions on the structure of CSP
  instances that yield tractable classes have been
  identified. However, many such restrictions fail to guarantee
  tractability for CSPs with global constraints. In this paper, we
  investigate the properties of extensionally represented constraints
  that these restrictions exploit to achieve tractability, and show
  that there are large classes of global constraints that also possess
  these properties. This allows us to lift these restrictions to the
  global case, and identify new tractable classes of CSPs with global
  constraints.
\end{abstract}

\section{Introduction}

Constraint programming (CP) is widely used to solve a variety of
practical problems such as planning and scheduling
\cite{vanHoeve2006:global-const,Wallace96practicalapplications}, and
industrial configuration
\cite{pup-cpaior2011,bin-repack-datacenters}. Constraints can either
be represented explicitly, by a table of allowed assignments, or
implicitly, by specialized algorithms provided by the constraint
solver. These algorithms may take as a parameter a \emph{description}
that specifies exactly which kinds of assignments a particular
instance of a constraint should allow. Such implicitly represented
constraints are known as global constraints, and a lot of the success
of CP in practice has been attributed to solvers providing them
\cite{Rossi06:handbook,Gent06:minion,wallace97-eclipse}.

The theoretical properties of constraint problems, in particular the
computational complexity of different types of problem, have been
extensively studied and quite a lot is known about what restrictions
on the general \emph{constraint satisfaction problem} are sufficient
to make it tractable
\cite{struct-decomp-stacs-et,bulatov05:classifying-constraints,CohenUnifTract,compar-decomp-csp,grohe-hom-csp-complexity,stoc-Marx10-submod-width}. In
particular, many structural restrictions, that is, restrictions on how
the constraints in a problem interact, have been identified and shown
to yield tractable classes of CSP instances
\cite{Gottlob02:jcss-hypertree,grohe-marx-frac-edge-cover,stoc-Marx10-submod-width}. However,
much of this theoretical work has focused on problems where each
constraint is explicitly represented, and most known structural
restrictions fail to yield tractable classes for problems with global
constraints, even when the global constraints are fairly simple
\cite{Kutz08:sim-match}.

Theoretical work on global constraints has to a large extent focused
on developing efficient algorithms to achieve various kinds of local
\emph{consistency} for individual constraints.  This is generally done
by pruning from the domains of variables those values that cannot lead
to a satisfying assignment
\cite{Walsh07:global,Samer11:constraints-tractable}.  Another strand
of research has explored conditions that allow global
constraints to be replaced by collections of explicitly represented constraints
\cite{bess-nvalue}. These techniques allow faster implementations of
algorithms for \emph{individual constraints}, but do not shed much
light on the complexity of problems with multiple \emph{overlapping}
global constraints, which is something that practical problems
frequently require.

As such, in this paper we investigate what properties of explicitly
represented constraints structural restrictions rely on to guarantee
tractability. Identifying such properties will allow us to find global
constraints that also possess them, and lift well-known
structural restrictions to instances with such constraints.

As discussed in~\cite{ChenGroheCompactRel}, when the constraints in a
family of problems have unbounded arity, the way that the constraints
are {\em represented} can significantly affect the complexity.
Previous work in this area has assumed that the global constraints
have specific representations, such as
propagators~\cite{Green08:cp-structural}, negative
constraints~\cite{gdnf-cohen-repr}, or GDNF/decision
diagrams~\cite{ChenGroheCompactRel}, and exploited properties
particular to that representation. In contrast, we will use a
definition of global constraints that allows us to discuss different
representations in a uniform manner. Furthermore, as the results we
obtain will rely on a relationship between the size of a global
constraint and the number of its satisfying assignments, we do not
need to reference any specific representation. 

As a running example, we will use the connected graph partition
problem (CGP) \cite[p.~209]{Garey79:intractability}, defined
below. The CGP is the problem of partitioning the vertices of a graph
into bags of a given size while minimizing the number of edges that
span bags. The vertices of the graph could represent components to be
placed on circuit boards while minimizing the number of inter-board
connections.

\begin{problem}[Connected graph partition (CGP)]
  \label{prob:cgp}
  We are given an undirected and connected graph $\tup{V, E}$, as well
  as $\alpha, \beta \in \mathbb N$. Can $V$ be partitioned into
  disjoint sets $V_1,\ldots,V_m$ with $|V_i| \leq \alpha$ such that
  the set of broken edges $E' = \{\{u, v\} \in E \mid u \in V_i, v \in
  V_j, i\not=j\}$ has cardinality $\beta$ or less?
\end{problem}

This problem is \NP-complete \cite[p.~209]{Garey79:intractability},
even for fixed $\alpha \geq 3$. We are going to use the results in
this paper to show a new result, namely that the CGP is tractable for
every fixed $\beta$.

\section{Global Constraints}

In this section, we define the basic concepts that we will use
throughout the paper. In particular, we give a precise definition of
global constraints, and illustrate it with a few examples.

\begin{definition}[Variables and assignments]
  Let $V$ be a set of variables, each with an associated set of domain
  elements. We denote the set of domain elements (the domain) of a
  variable $v$ by $D(v)$.
  We extend this notation to arbitrary subsets of variables, $W$, 
  by setting $D(W) = \displaystyle\bigcup_{v \in W} D(v)$.

  An {\em assignment} of a set of variables $V$ is a function $\theta
  : V \rightarrow D(V)$ that maps every $v \in V$ to an element
  $\theta(v) \in D(v)$.  We denote the restriction of $\theta$ to a
  set of variables $W \subseteq V$ by $\theta|_W$. We also allow the
  special assignment $\bot$ of the empty set of variables. In
  particular, for every assignment $\theta$, we have
  $\theta|_\emptyset = \bot$.
\end{definition}

\begin{definition}[Projection]
  Let $\Theta$ be a set of assignments of a set of variables $V$. The
  \emph{projection} of $\Theta$ onto a set of variables $X \subseteq
  V$ is the set of assignments $\pi_X(\Theta) = \{\theta|_X \mid
  \theta \in \Theta\}$.
\end{definition}
   
Note that when $\Theta = \emptyset$ we have $\pi_X(\Theta) =
\emptyset$, but when $X = \emptyset$ and $\Theta \neq \emptyset$, we
have $\pi_X(\Theta) = \{\bot\}$.

\begin{definition}[Disjoint union of assignments]
 \label{def:disjoint-union}
 Let $\theta_1$ and $\theta_2$ be two assignments of disjoint sets of
 variables $V_1$ and $V_2$, respectively. The \emph{disjoint union} of
 $\theta_1$ and $\theta_2$, denoted $\theta_1 \oplus \theta_2$, is the
 assignment of $V_1 \cup V_2$ such that $(\theta_1 \oplus \theta_2)(v)
 = \theta_1(v)$ for all $v \in V_1$, and $(\theta_1 \oplus
 \theta_2)(v) = \theta_2(v)$ for all $v \in V_2$.
\end{definition}

Global constraints have traditionally been defined, somewhat vaguely,
as constraints without a fixed arity, possibly also with a compact
representation of the constraint relation. For example, in
\cite{vanHoeve2006:global-const} a global constraint is defined as ``a
constraint that captures a relation between a non-fixed number of
variables''.

Below, we offer a precise definition similar to the one in
\cite{Walsh07:global}, where the authors define global constraints for
a domain $D$ over a list of variables $\sigma$ as being given
intensionally by a function $D^{|\sigma|} \rightarrow \{0, 1\}$
computable in polynomial time. Our definition differs from this one in
that we separate the general {\em algorithm} of a global constraint
(which we call its {\em type}) from the specific description.  This
separation allows us a better way of measuring the size of a global
constraint, which in turn helps us to establish new complexity
results.

\begin{definition}[Global constraints]
  \label{def:glob-const}
  A \emph{global constraint type} is a parameterized polynomial-time
  algorithm that determines the acceptability of an assignment of a
  given set of variables.

  Each global constraint type, $e$, has an associated set of 
  \emph{descriptions}, $\Delta(e)$. Each description $\delta \in \Delta(e)$
  specifies appropriate parameter values for the algorithm $e$. 
  In particular, each $\delta \in \Delta(e)$ specifies a set of
  variables, denoted by $\vars(\delta)$.

  A \emph{global constraint} $e[\delta]$, where $\delta \in
  \Delta(e)$, is a function that maps assignments of $\vars(\delta)$
  to the set $\{0,1\}$.  Each assignment that is allowed by
  $e[\delta]$ is mapped to 1, and each disallowed assignment is mapped
  to 0.  The \emph{extension} or \emph{constraint relation} of
  $e[\delta]$ is the set of assignments, $\theta$, of $\vars(\delta)$
  such that $e[\delta](\theta) = 1$. We also say that such assignments
  \emph{satisfy} the constraint, while all other assignments
  \emph{falsify} it.
\end{definition}

When we are only interested in describing the set of assignments that
satisfy a constraint, and not in the complexity of determining
membership in this set, we will sometimes abuse notation by writing
$\theta \in e[\delta]$ to mean $e[\delta](\theta) = 1$.

As can be seen from the definition above, a global constraint is not
usually explicitly represented by listing all the assignments that
satisfy it. Instead, it is represented by some description $\delta$
and some algorithm $e$ that allows us to check whether the constraint
relation of $e[\delta]$ includes a given assignment. To stay within
the complexity class \NP, this algorithm is required to run in
polynomial time. As the algorithms for many common global constraints
are built into modern constraint solvers, we measure the {\em size} of
a global constraint's representation by the size of its description.

\begin{example}[EGC]
  \label{example:egc}
  A very general global constraint type is the \emph{extended global
    cardinality} constraint type~\cite{Samer11:constraints-tractable}.
  This form of global constraint is defined by specifying for every
  domain element $a$ a finite set of natural numbers $K(a)$, called
  the cardinality set of $a$. The constraint requires that the number
  of variables which are assigned the value $a$ is in the set $K(a)$,
  for each possible domain element $a$.

  Using our notation, the description $\delta$ of an EGC global
  constraint specifies a function $K_\delta : D(\vars(\delta))
  \rightarrow \mathcal P(\mathbb N)$ that maps each domain element to
  a set of natural numbers.  The algorithm for the EGC constraint then
  maps an assignment $\theta$ to $1$ if and only if, for every domain
  element $a \in D(\vars(\delta))$, we have that $|\{v \in
  \vars(\delta) \mid \theta(v) = a\}| \in K_\delta(a)$.

  % To simplify the notation, when specifying a description $\delta$ of
  % a counting constraint, for any domain element $a$ not mentioned we
  % assume $\delta(a) = \{0\}$.
\end{example}

% \begin{example}[Clauses]
%   \label{example:clauses}
%   We can view the disjunctive clauses used to define propositional satisfiability
%   problems as a global constraint type in the following way.
  
%   The description $\delta$ of a clause is simply a list of the literals that it contains,
%   and $\vars(\delta)$ is the corresponding set of variables.
%   The algorithm for the clause then maps any Boolean assignment 
%   $\theta$ of $\vars(\delta)$ that
%   satisfies the disjunction of the literals in $\delta$ to 1, 
%   and all other assignments to 0.
  
%   Note that a clause forbids precisely one assignment to
%   $\vars(\delta)$ (the one that falsifies all of the literals in the
%   clause).  Hence the extension of a clause contains
%   $2^{|\vars(\delta)|}-1$ assignments, so the size of the constraint
%   relation grows exponentially with the number of variables, but the
%   size of the constraint description grows only linearly.
%  \end{example}

\begin{example}[Table and negative constraints]
  \label{example:table-const}
  A rather degenerate example of a a global constraint type is 
  the \emph{table} constraint.
  
  In this case the description $\delta$ is simply a list of assignments 
  of some fixed set of variables, $\vars(\delta)$. The algorithm for
  a table constraint then decides, for any
  assignment of $\vars(\delta)$, whether it is included in $\delta$. 
  This can be done in a time which is linear in the size of $\delta$ 
  and so meets the polynomial time requirement. 

  {\em Negative} constraints are complementary to table constraints,
  in that they are described by listing {\em forbidden}
  assignments. The algorithm for a negative constraint $e[\delta]$
  decides, for any assignment of $\vars(\delta)$, whether whether it
  is {\em not} included in $\delta$. Observe that disjunctive clauses,
  used to define propositional satisfiability problems, are a special
  case of the negative constraint type, as they have exactly one
  forbidden assignment.
  
  We observe that any global constraint can be rewritten as a table or
  negative constraint. However, this rewriting will, in general, incur
  an exponential increase in the size of the description.
\end{example}

As can be seen from the definition above, a table global constraint is
explicitly represented, and thus equivalent to the usual notion of an
explicitly represented constraint.

\begin{definition}[CSP instance]
  An instance of the constraint satisfaction problem (CSP) is a pair
  $\tup{V, C}$ where $V$ is a finite set of \emph{variables}, and $C$
  is a set of \emph{global constraints} such that for every $e[\delta]
  \in C$, $\vars(\delta) \subseteq V$.  In a CSP instance, we call
  $\vars(\delta)$ the \emph{scope} of the constraint $e[\delta]$.

  A \emph{classic} CSP instance is one where every constraint is a
  table constraint.

  A \emph{solution} to a CSP instance $P = \tup{V, C}$ is an
  assignment $\theta$ of $V$ which satisfies every global constraint,
  i.e., for every $e[\delta] \in C$ we have $\theta|_{\vars(\delta)}
  \in e[\delta]$. We denote the set of solutions to $P$ by $\sol(P)$.

  The \emph{size} of a CSP instance $P = \tup{V, C}$ is $|P| = |V| +
  \displaystyle\sum_{v \in V} |D(v)| + \displaystyle\sum_{e[\delta]
    \in C} |\delta|$.
\end{definition}

\begin{example}[The CGP encoded with global constraints]
  \label{example:cgp-as-csp}
  Given a connected graph $G = \tup{V, E}$, $\alpha$, and $\beta$, we
  build a CSP instance $\tup{A \cup B, C}$ as follows. The set $A$
  will have a variable $v$ for every $v \in V$ with domain $D(v) =
  \{1,\ldots,|V|\}$, while the set $B$ will have a boolean variable
  $e$ for every edge in $E$.

  The set of constraints $C$ will have an EGC constraint $C^{\alpha}$
  on $A$ with $K(i) = \{0,\ldots,\alpha\}$ for every $1 \leq i \leq
  |V|$. Likewise, $C$ will have an EGC constraint $C^\beta$ on $B$
  with $K(0) = \{0,\ldots,|E|\}$ and $K(1) = \{1,\ldots,\beta\}$.

  Finally, to connect $A$ and $B$, the set $C$ will have for every
  edge $\{u, v\} \in E$, with corresponding variable $e \in B$, a
  table constraint on $\{u, v, e\}$ requiring $u \not= v \rightarrow
  e=1$.

  As an example, \cref{fig:cgp-ext-example} shows this encoding for
  the CGP on the graph $C_5$, that is, a simple cycle on five
  vertices.
\end{example}

This encoding follows the definition of \cref{prob:cgp} quite closely,
and can be done in polynomial time.

\begin{figure}[h]
  \centering
  \includegraphics[width=0.8\textwidth]{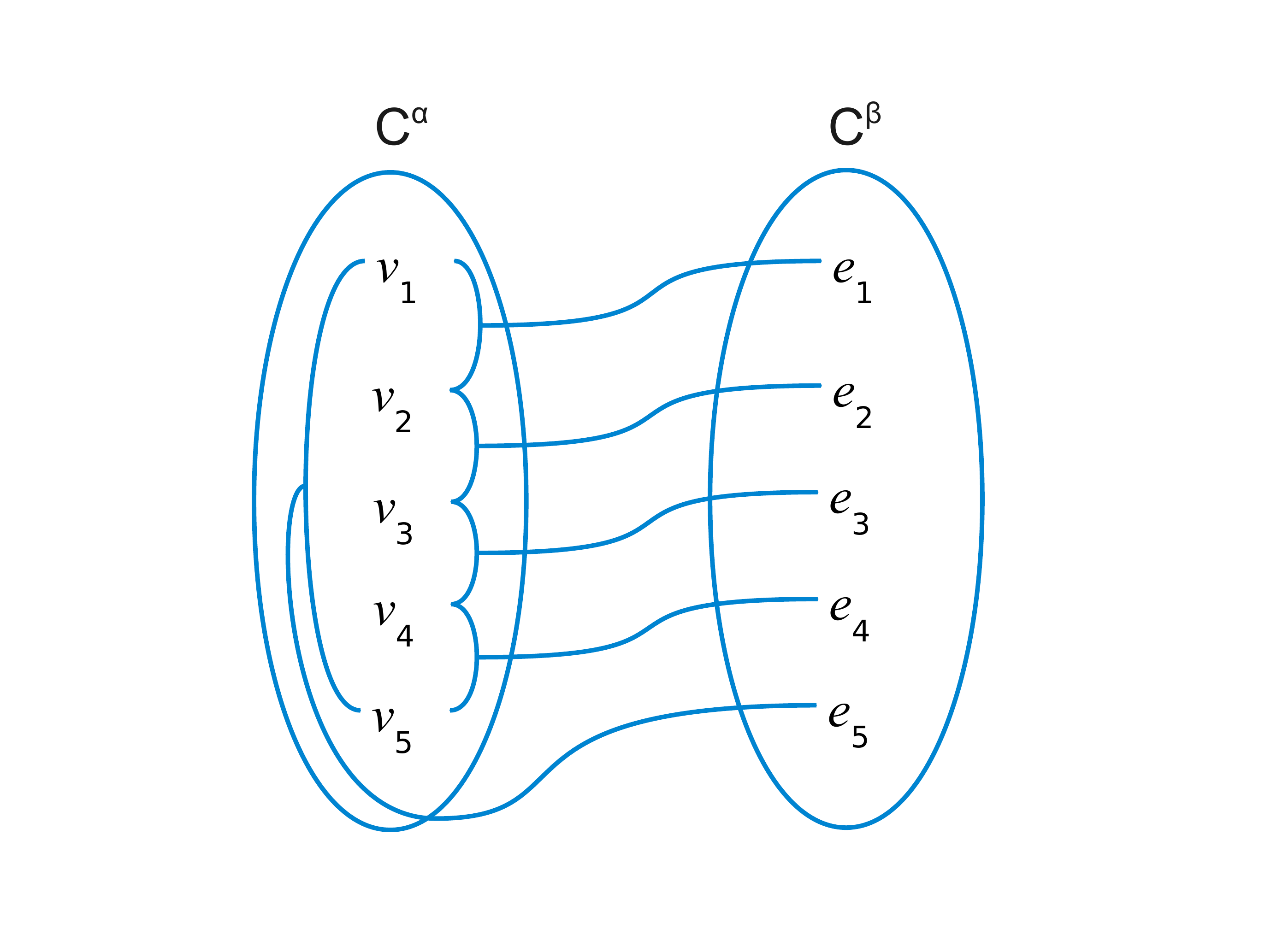}
  \caption{CSP encoding of the CGP on the graph $C_5$.}
  \label{fig:cgp-ext-example}
\end{figure}

\section{Structural Restrictions}

% Hypertree, frac. hypertree, and submod. width

In recent years, there has been a flurry of research into identifying
tractable classes of classic CSP instances based on restrictions on
the hypergraphs of CSP instances, known as structural restrictions.
Below, we present and discuss a few representative examples. To
present the various structural restrictions, we will use the framework
of width functions, introduced by Adler \cite{adler-thesis}.

\begin{definition}[Hypergraph]
  A hypergraph $\tup{V, H}$ is a set of vertices $V$ together with a
  set of hyperedges $H \subseteq \mathcal P(V)$.

  Given a CSP instance $P = \tup{V, C}$, the hypergraph of $P$,
  denoted $\hyp(P)$, has vertex set $V$ together with a hyperedge
  $\vars(\delta)$ for every $e[\delta] \in C$.
\end{definition}

\begin{definition}[Tree decomposition]
  A \emph{tree decomposition} of a hypergraph $\tup{V, H}$ is a pair
  $\langle T, \lambda \rangle$ where $T$ is a tree and $\lambda$ is a
  labelling function from nodes of $T$ to subsets of $V$, such that
  \begin{enumerate}
  \item for every $v \in V$, there exists a node $t$ of $T$ such that
    $v \in \lambda(t)$,
  \item for every hyperedge $h \in H$, there exists a node $t$ of $T$
    such that $h \subseteq \lambda(t)$, and
  \item for every $v \in V$, the set of nodes $\{ t \mid v \in
    \lambda(t) \}$ induces a connected subtree of $T$.
  \end{enumerate}
\end{definition}

\begin{definition}[Width function]
  Let $G = \tup{V, H}$ be a hypergraph. A \emph{width function} on $G$
  is a function $f : \mathcal P(V) \rightarrow \mathbb R^+$ that
  assigns a positive real number to every nonempty subset of vertices
  of $G$. A width function $f$ is monotone if $f(X) \leq f(Y)$
  whenever $X \subseteq Y$.

  Let $\tup{T, \lambda}$ be a tree decomposition of $G$, and $f$ a
  width function on $G$. The \emph{$f$-width} of $\tup{T, \lambda}$ is
  $\max(\{f(\lambda(t)) \mid t \mbox{ node of } T\})$. The
  \emph{$f$-width} of $G$ is the minimal $f$-width over all its tree
  decompositions.
\end{definition}

In other words, a width function on a hypergraph $G$ tells us how to
assign weights to nodes of tree decompositions of $G$.

\begin{definition}[Treewidth]
  Let $f(X) = |X|-1$. The treewidth $\tw(G)$ of a hypergraph $G$ is
  the $f$-width of $G$.
\end{definition}

Let $G = \tup{V, H}$ be a hypergraph, and $X \subseteq V$. An edge
cover for $X$ is any set of hyperedges $H' \subseteq H$ that satisfies
$X \subseteq \bigcup H'$. The edge cover number $\rho(X)$ of $X$ is
the size of the smallest edge cover for $X$. It is clear that $\rho$
is a width function.

\begin{definition}[\protect{\cite[Chapter 2]{adler-thesis}}]
  The generalized hypertree width $\hw(G)$ of a hypergraph $G$ is the
  $\rho$-width of $G$.
\end{definition}

Next, we define a relaxation of hypertree width known as fractional
hypertree width, introduced by Grohe and Marx
\cite{grohe-marx-frac-edge-cover}.

\begin{definition}[Fractional edge cover]
  Let $G = \tup{V, H}$ be a hypergraph, and $X \subseteq V$. A
  \emph{fractional edge cover} for $X$ is a function $\gamma : H
  \rightarrow [0, 1]$ such that $\displaystyle\sum_{v \in h \in H}
  \gamma(h) \geq 1$ for every $v \in X$. We call $\displaystyle\sum_{h
    \in H} \gamma(h)$ the weight of $\gamma$. The \emph{fractional
    edge cover number} $\rho^*(X)$ of $X$ is the minimum weight over
  all fractional edge covers for $X$. It is known that this minimum is
  always rational \cite{grohe-marx-frac-edge-cover}.
\end{definition}

\begin{definition}
  The \emph{fractional hypertree width} $\fhw(G)$ of a hypergraph $G$
  is the $\rho^*$-width of $G$.
\end{definition}

For a class of hypergraphs $\mathcal H$ and a notion of width
$\alpha$, we write $\alpha(\mathcal H)$ for the maximal $\alpha$-width
over the hypergraphs in $\mathcal H$. If this is unbounded we write
$\alpha(\mathcal H) = \infty$; otherwise $\alpha(\mathcal H) <
\infty$.

All the above restrictions can be used to guarantee tractability for
classes of CSP instances where all constraints are table
constraints.

\begin{theorem}[\cite{Dalmau02:csp-tractability-cores,Gottlob02:jcss-hypertree,grohe-marx-frac-edge-cover}]
  \label{thm:width-tract}
  Let $\mathcal H$ be a class of hypergraphs. For every $\alpha \in
  \{\hw,\fhw\}$, any class of classic CSP instances whose
  hypergraphs are in $\mathcal H$ is tractable if $\alpha(\mathcal H)
  < \infty$.
\end{theorem}

To go beyond fractional hypertree width, Marx
\cite{stoc-Marx10-submod-width,DBLP:journals/corr/abs-0911-0801}
recently introduced the concept of submodular width. This concept uses
a set of width functions satisfying a condition (submodularity), and
considers the $f$-width of a hypergraph for every such function $f$.

\begin{definition}[Submodular width function]
  Let $G = \tup{V, H}$ be a hypergraph. A width function $f$ on $G$ is
  \emph{submodular} if for every set $X, Y \subseteq V$, we have $f(X)
  + f(Y) \geq f(X \cap Y) + f(X \cup Y)$.
\end{definition}

\begin{definition}[Submodular width]
  Let $G$ be a hypergraph. The \emph{submodular width} $\subw(G)$ of
  $G$ is the maximum $f$-width of $G$ taken over all monotone
  submodular width functions $f$ on $G$.

  For a class of hypergraphs $\mathcal H$, we write $\subw(\mathcal
  H)$ for the maximal submodular width over the hypergraphs in
  $\mathcal H$. If this is unbounded we write $\subw(\mathcal H) =
  \infty$; otherwise $\subw(\mathcal H) < \infty$.
\end{definition}

Unlike for fractional hypertree width and every other structural
restriction discussed so far, the running time of the algorithm given
by Marx for classic CSP instances with bounded submodular width has an
exponential dependence on the number of vertices in the hypergraph of
the instance. The class of classic CSP instances with bounded
submodular width is therefore not tractable. However, this class is
what is called fixed-parameter tractable
\cite{downey2012parameterized,flum2006parameterized}.

% \begin{theorem}[\cite{DBLP:journals/corr/abs-0911-0801}]
%   \label{thm:submodular-width}
%   Let $P$ be a classic CSP instance with $\subw(\hyp(P)) =
%   k$. Deciding whether $P$ has a solution can be done in time $O(2^{k
%     \times 2^{O(n)}} \times |P|^{O(k)})$, where $n$ is the number of
%   vertices of $\hyp(P)$.
% \end{theorem}

% As we can see, the running time has an exponential dependence on the
% number of vertices in the hypergraph of the instance. The class of
% classic CSP instances with bounded submodular width is therefore not
% tractable. However, this class is what is called fixed-parameter
% tractable \cite{downey2012parameterized,flum2006parameterized}.

\begin{definition}[Fixed-parameter tractable]
  A \emph{parameterized problem instance} is a pair $\tup{k, P}$,
  where $P$ is a problem instance, such as a CSP instance, and $k \in
  \mathbb N$ a parameter.

  Let $S$ be a class of parameterized problem instances. We say that
  $S$ is \emph{fixed-parameter tractable} (in \FPT) if there is a
  function $f$ of one argument, as well as a constant $c$, such that
  every problem $\tup{k, P} \in S$ can be solved in time $O(f(k)
  \times |P|^c)$.
\end{definition}

The function $f$ can be arbitrary, but must only depend on the
parameter $k$. For CSP instances, a natural parameterization is by the
size of the hypergraph of an instance, measured by the number of
vertices. Since the hypergraph of an instance has a vertex for every
variable, for every CSP instance $P = \tup{V, C}$ we consider the
parameterized instance $\tup{|V|, P}$.

\begin{theorem}[\cite{DBLP:journals/corr/abs-0911-0801}]
  \label{cor:subw}
  Let $\mathcal H$ be a class of hypergraphs. If $\subw(\mathcal H) <
  \infty$, then a class of classic CSP instances whose hypergraphs are
  in $\mathcal H$ is in \FPT{}.
\end{theorem}

The three structural restrictions that we have just presented form a
hierarchy
\cite{grohe-marx-frac-edge-cover,DBLP:journals/corr/abs-0911-0801}:
For every hypergraph $G$, $\subw(G) \leq \fhw(G) \leq \hw(G)$.

As the example below demonstrates, \cref{thm:width-tract} does not
hold for CSP instances with arbitrary global constraints, even if we
have a fixed, finite domain.

\begin{example}
  \label{example:3col}
  The \NP-complete problem of 3-colourability
  \cite{Garey79:intractability} is to decide, given a graph
  $\tup{V,E}$, whether the vertices $V$ can be coloured with three
  colours such that no two adjacent vertices have the same colour.

  We may reduce this problem to a CSP with EGC constraints
  (cf.~Example~\ref{example:egc}) as follows: Let $V$ be the set of
  variables for our CSP instance, each with domain $\{r,g,b\}$. For
  every edge $\tup{v, w} \in E$, we post an EGC constraint with scope
  $\{v, w\}$, parameterized by the function $K$ such that $K(r) = K(g)
  = K(b) = \{0,1\}$. Finally, we make the hypergraph of this CSP
  instance have low width by adding an EGC constraint with scope $V$
  parameterized by the function $K'$ such that $K'(r) = K'(g) = K'(b)
  = \{0,\ldots,|V|\}$. This reduction clearly takes polynomial time,
  and the hypergraph $G$ of the resulting instance has
  $\hw(G)=\fhw(G)=\subw(G)=1$.

  As the constraint with scope $V$ allows
  all possible assignments, any solution to this CSP is also a
  solution to the 3-colourability problem, and vice versa.

  % Hence we have shown that, when the constraints are represented intensionally
  % by global constraints, an NP-complete problem can be reduced in polynomial time to 
  % a class of CSP instances with acyclic hypergraphs.
\end{example}

Likewise, \cref{cor:subw} does not hold for CSP instances with
arbitrary global constraints if we allow the variables unbounded
domain size, that is, change the above example to
$k$-colourability. With that in mind, in the rest of the paper we will
identify properties of extensionally represented constraints that
these structural restrictions exploit to guarantee tractability. Then,
we are going to look for restricted classes of global constraints that
possess these properties. To do so, we will use the following
definitions.

\begin{definition}[Constraint \languageword]
  A \emph{constraint \languageword} is a set of global constraints. A
  CSP instance $\tup{V, C}$ is said to be over a constraint
  \languageword\ $\languagesymbol$ if for every $e[\delta] \in C$ we
  have $e[\delta] \in \languagesymbol$.
\end{definition}

\begin{definition}[Restricted CSP class]
\label{def:CSPrestricted}
Let $\languagesymbol$ be a constraint \languageword, and let $\mathcal
H$ be a class of hypergraphs.  We define $\CSP(\mathcal
H,\languagesymbol)$ to be the class of CSP instances over
$\languagesymbol$ whose hypergraphs are in $\mathcal H$.

  % If $\mathcal A$ is the set of all structures, we denote that by the
  % shorthand $-$, and do likewise for $\languagesymbol$ if it includes all
  % constraint representations.
\end{definition}

\Cref{def:CSPrestricted} allows us to discuss classic CSP instances
alongside instances with global constraints. Let $\Extlang$ be the
constraint \languageword{} containing all table global
constraints. The classic CSP instances are then precisely those that
are over $\Extlang$. In particular, we can now restate
\cref{thm:width-tract,cor:subw} as follows.

\begin{theorem}
  \label{thm:width-new-notation}
  Let $\mathcal H$ be a class of hypergraphs. For every $\alpha \in
  \{\hw,\fhw\}$, the class of CSP instances $\CSP(\mathcal H,
  \Extlang)$ is tractable if $\alpha(\mathcal H) <
  \infty$. Furthermore, if $\subw(\mathcal H) < \infty$ then
  $\CSP(\mathcal H, \Extlang)$ is in \FPT.
\end{theorem}

\section{Properties of Extensional Representation}

We are going to start our investigation by considering fractional
hypertree width in more detail. To obtain tractability for classic CSP
instances of bounded fractional hypertree width, Grohe and Marx
\cite{grohe-marx-frac-edge-cover} use a bound on the number of
solutions to a classic CSP instance, and show that this bound is
preserved when we consider parts of a CSP instance. The following
definition formalizes what we mean by ``parts'', and is required to
state the algorithm that Grohe and Marx use in their paper.

\begin{definition}[Constraint projection]
  \label{def:constraint-projection}
  Let $e[\delta]$ be a constraint. The \emph{projection of
    $e[\delta]$} onto a set of variables $X \subseteq \vars(\delta)$
  is the constraint $\Pj_X(e[\delta])$ such that $\mu \in
  \Pj_X(e[\delta])$ if and only if there exists $\theta \in e[\delta]$
  with $\theta|_X = \mu$.

  For a CSP instance $P = \tup{V, C}$ and $X \subseteq V$ we define
  $\Pj_X(P) = \tup{X, C'}$, where $C'$ is the least set containing for
  every $e[\delta] \in C$ such that $X \cap \vars(\delta) \not=
  \emptyset$ the constraint $\Pj_{X \cap \vars(\delta)}(e[\delta])$.
\end{definition}

Their algorithm is given as \cref{alg:enum-solutions}, and is
essentially the usual recursive search algorithm for finding all
solutions to a CSP instance by considering smaller and smaller
sub-instances using constraint projections.

\begin{algorithm}[h]
  \begin{algorithmic}
    \Procedure{EnumSolutions}{CSP instance $P = \tup{V, C}$}\Comment{Returns $\sol(P)$}
    \State{$\textup{Solutions} \leftarrow \emptyset$}
    \If {$V = \emptyset$}
      \State\Return{$\{\bot\}$}\Comment{The empty assignment}
    \Else
      \State{$w \leftarrow \textup{chooseVar}(V)$}\Comment{Pick a variable from $V$}
      \State{$\Theta = \textup{EnumSolutions}(\Pj_{V-\{w\}}(P))$}
      \For{$\theta \in \Theta$}
        \For{$a \in D(w)$}
          \If{$\theta \cup \tup{w, a}$ is a solution to $P$}
            \State $\textup{Solutions.add}(\theta \cup \tup{w, a})$
          \EndIf
        \EndFor
      \EndFor
    \EndIf
    \State\Return Solutions
    \EndProcedure
  \end{algorithmic}
  \caption{Enumerate all solutions of a CSP instance}
  \label{alg:enum-solutions}
\end{algorithm}

To show that \cref{alg:enum-solutions} does indeed find all solutions,
we will use the following property of constraint projections.

\begin{lemma}
  \label{lemma:solution-projection}
  Let $P = \tup{V, C}$ be a CSP instance. For every $X \subseteq V$,
  we have $\sol(\Pj_X(P)) \supseteq \pi_X(\sol(P))$.
\end{lemma}
\begin{proof}
  Given $P = \tup{V, C}$, let $X \subseteq V$ be arbitrary, and
  let $C' = \{e[\delta] \in C \mid X \cap \vars(\delta) \not=
  \emptyset\}$. For every $\theta \in \sol(P)$ and constraint
  $e[\delta] \in C'$ we have that $\theta|_{\vars(\delta)} \in
  e[\delta]$ since $\theta$ is a solution to $P$. By
  \cref{def:constraint-projection}, it follows that for every
  $e[\delta] \in C'$, $\theta|_{X \cap \vars(\delta)} \in \Pj_{X \cap
    \vars(\delta)}(e[\delta])$. Since the set of constraints of
  $\Pj_X(P)$ is the least set containing for each $e[\delta] \in C'$
  the constraint $\Pj_{X \cap \vars(\delta)}(e[\delta])$, we have
  $\theta|_X \in \sol(\Pj_X(P))$, and hence $\sol(\Pj_X(P)) \supseteq
  \pi_X(\sol(P))$. Since $X$ was arbitrary, the claim follows.
\end{proof}

% Using \cref{lemma:solution-projection}, we can now show that
% \cref{alg:enum-solutions} is correct, that is, finds all solutions to
% a CSP instance.

\begin{theorem}[Correctness of \cref{alg:enum-solutions}]
  Let $P$ be a CSP instance. We have that
  $\textup{EnumSolutions}(P) = \sol(P)$.
\end{theorem}
\begin{proof}
  The proof is by induction on the set of variables $V$ in $P$. For
  the base case, if $V = \emptyset$, the empty assignment is the only
  solution.

  Otherwise, choose a variable $w \in V$, and let $X = V
  - \{w\}$. By induction, we can assume that
  $\textup{EnumSolutions}(\Pj_X(P)) = \sol(\Pj_X(P))$. Since for
  every $\theta \in \sol(P)$ there exists $a \in D(w)$ such that
  $\theta = \theta|_X \cup \tup{w, a}$, and furthermore $\theta|_X \in
  \pi_X(\sol(P))$, it follows by \cref{lemma:solution-projection} that
  $\theta|_X \in \sol(\Pj_X(P))$. Since \cref{alg:enum-solutions}
  checks every assignment of the form $\mu \cup \tup{w, a}$ for every
  $\mu \in \sol(\Pj_X(P))$ and $a \in D(w)$, it follows that
  $\textup{EnumSolutions}(P) = \sol(P)$.
\end{proof}

The time required for this algorithm depends on three key factors,
which we are going to enumerate and discuss below. Let
\begin{enumerate}
\item $s(P)$ be the maximum of the number of solutions to each of the
  instances $\Pj_{V-\{w\}}(P)$,
\item $c(P)$ be the maximum time required to check whether an
  assignment is a solution to $P$, and
\item $b(P)$ be the maximum time required to construct any instance
  $\Pj_{V-\{w\}}(P)$.
\end{enumerate}

There are $|V|$ calls to $\textup{EnumSolutions}$. For
each call, we need $b(P)$ time to construct the projection, while the
double loop takes at most $s(P) \times |D(w)| \times c(P)$
time. Therefore, letting $d = \max(\{|D(w)| \mid w \in V\})$,
the running time of \cref{alg:enum-solutions} is bounded by
$O\big(|V| \times (s(P) \times d \times c(P) + b(P))\big)$.

Since constructing the projection of a classic CSP instance can be
done in polynomial time, and likewise checking that an assignment is a
solution, the whole algorithm runs in polynomial time if $s(P)$ is a
polynomial in the size of $P$. For fractional hypertree width, Grohe
and Marx show the following.

\begin{lemma}[\cite{grohe-marx-frac-edge-cover}]
  \label{lemma:fhw-few-solutions}
  A classic CSP instance $P$ has at most $|P|^{\fhw(\hyp(P))}$
  solutions.
\end{lemma}

% The proof uses a combinatorial result about hypergraphs, known as
% Shearer's Lemma \cite{chung1986some}, to bound the number of
% solutions.
Since fractional hypertree width is a monotone width function, it
follows that for any instance $P = \tup{V, C}$ and $X \subseteq V$,
$\fhw(\hyp(\Pj_X(P))) \leq \fhw(\hyp(P))$. Therefore, for classic CSP
instances of bounded fractional hypertree width $s(P)$ is indeed
polynomial in $|P|$.

\section{CSP Instances with Few Solutions in Key Places}

Having few solutions for every projection of a CSP instance is thus a
property that makes fractional hypertree width yield tractable classes
of classic CSP instances. More importantly, we have shown that this
property allows us to find all solutions to a CSP instance $P$, even
with global constraints, if we can build arbitrary projections of $P$
in polynomial time. In other words, with these two conditions we
should be able to reduce instances with global constraints to classic
instances in polynomial time.

However, on reflection there is no reason why we should need few
solutions for \emph{every} projection. Instead, consider the following
reduction.

\begin{definition}[Partial assignment checking]
  \label{def:part-assignment-checking}
  A global constraint \languageword{} $\Gamma$ allows \emph{partial
    assignment checking} if for any constraint $e[\delta] \in \Gamma$
  we can decide in polynomial time whether a given assignment $\theta$
  to a set of variables $W \subseteq \vars(\delta)$ is contained in an
  assignment that satisfies $e[\delta]$, i.e.~whether there exists
  $\mu \in e[\delta]$ such that $\theta = \mu|_W$.
\end{definition}

As an example, a \languageword{} that contains arbitrary EGC
constraints (cf.~\cref{example:egc}) does not satisfy
\cref{def:part-assignment-checking}, since checking whether an
arbitrary EGC constraint has a satisfying assignment is \NP-hard
\cite{quimper04-gcc-npc}. On the other hand, a \languageword{} that
contains only EGC constraints whose cardinality sets are intervals
does satisfy \cref{def:part-assignment-checking}
\cite{Regin96:aaai-generalized}.

If a \languageword{} $\Gamma$ satisfies
\cref{def:part-assignment-checking}, we can for any constraint
$e[\delta] \in \Gamma$ build arbitrary projections of it, that is,
construct the global constraint $\Pj_X(e[\delta])$ for any $X
\subseteq \vars(\delta)$, in polynomial time.
% As \cref{alg:enum-solutions} requires us to construct projections,
% this is a useful property.

\begin{definition}[Intersection variables]
  \label{def:intersection-vertices}
  Let $\tup{V, C}$ be a CSP instance. The set of \emph{intersection
    variables} of any constraint $e[\delta] \in P$ is $\iv(\delta) =
  \bigcup \{ \vars(\delta) \cap \vars(\delta') \mid e'[\delta'] \in C
  - \{e[\delta]\}\}$.
\end{definition}

\begin{definition}[Table constraint induced by a global constraint]
  \label{def:ind-const}
  Let $P = \tup{V, C}$ be a CSP instance. For every $e[\delta] \in C$,
  let $\mu^*$ be the assignment to $\vars(\delta) - \iv(\delta)$ that
  assigns a special value $*$ to every variable. The \emph{table
    constraint induced by $e[\delta]$} is $\ic(e[\delta]) =
  e'[\delta']$, where $\vars(\delta') = \vars(\delta)$, and $\delta'$
  contains for every assignment $\theta \in
  \sol(\Pj_{\iv(\delta)}(P))$ the assignment $\theta \oplus \mu^*$.
\end{definition}

If every constraint in a CSP instance $P = \tup{V, C}$ allows partial
assignment checking, then building $\ic(e[\delta])$ for any $e[\delta]
\in C$ can be done in polynomial time when $|\sol(\Pj_{X}(P))|$ is
itself polynomial in the size of $P$ for every subset $X$ of
$\iv(\delta)$. To do so, we can invoke \cref{alg:enum-solutions} on
the instance $\Pj_{\iv(\delta)}(P)$. The definition below expresses
this idea.

\begin{definition}[Sparse intersections]
  \label{def:sparse-intersections}
  A class of CSP instances $\mathcal P$ \emph{has sparse
    intersections} if there exists a constant $c$ such that for every
  constraint $e[\delta]$ in any instance $P \in \mathcal P$, we have
  that for every $X \subseteq \iv(\delta)$, $|\sol(\Pj_X(P))| \leq
  |P|^c$.
\end{definition}

If a class of instances $\mathcal P$ has sparse intersections, and the
instances are all over a constraint \languageword{} that allows
partial assignment checking, then we can for every constraint
$e[\delta]$ of any instance from $\mathcal P$ construct
$\ic(e[\delta])$ in polynomial time. While this definition considers
the instance as a whole, one special case of it is the case where
every constraint has few solutions in the size of its description,
that is, there is a constant $c$ and the constraints are drawn from a
\languageword{} $\Gamma$ such that for every $e[\delta] \in \Gamma$,
we have that $|\{\mu \mid \mu \in e[\delta]\}| \leq |\delta|^c$.

% \begin{example}
%   \label{example:sparse-intersections}
%   Recall the second family of subproblem decompositions in
%   \cref{example:subproblem-decompositions-simple}. For a decomposition
%   $S = \{P, Q\}$ from this family, the set of intersection vertices
%   for both subproblems is $\{y_1,\ldots,y_n\}$. Furthermore, the EGC
%   constraint $A$ requires that there are exactly $4$ variables
%   assigned $1$ among $\{x_1,\ldots,x_n\}$, so there are $\binom{n}{4}$
%   satisfying assignments for this constraint. The equality constraints
%   ensure that this is the number of solutions to the whole subproblem
%   $P$, so for every $X \subseteq \{y_1,\ldots,y_n\}$ we have that
%   $|\sol(\Pj_X(S))| \leq \binom{n}{4}$. Therefore, this family of
%   subproblem decompositions has sparse intersections.
% \end{example}

% Another example where \cref{def:sparse-intersections} would be
% satisfied is when every set of intersection vertices is covered by a
% fixed number of table constraints. In this case, the number of
% possible solutions is bounded by the size of the join of all these
% constraints (cf.~\cref{sect:structural-restrictions}). This is the
% condition used by Cohen and Green \cite{guarded-decomp-CG}. In other
% words, we can derive the main theorem in as a special case of
% ours. We, however, do not need to cover the intersection vertices of
% \emph{every} subproblem by a fixed number of table constraints.

\begin{theorem}
  \label{thm:decomp-to-classic}
  Let $\mathcal P$ be a class of CSP instances over a \languageword{}
  that allows partial assignment checking. If $\mathcal P$ has sparse
  intersections, then we can in polynomial time reduce any instance $P
  \in \mathcal P$ to a classic CSP instance $P_{CL}$ with $\hyp(P) =
  \hyp(P_{CL})$, such that $P_{CL}$ has a solution if and only if $P$
  does.
\end{theorem}
\begin{proof}
  Let $P=\tup{V, C}$ be an instance from such a class $\mathcal
  P$. For each $e[\delta] \in C$, $P_{CL}$ will contain the table
  constraint $\ic(e[\delta])$ from \cref{def:ind-const}. Since $P$ is
  over a \languageword{} that allows partial assignment checking, and
  $\mathcal P$ has sparse intersections, computing $\ic(e[\delta])$
  can be done in polynomial time by invoking \cref{alg:enum-solutions}
  on $\Pj_{\iv(\delta)}(P)$.

  It is clear that $\hyp(P) = \hyp(P_{CL})$. All that is left to show
  is that $P_{CL}$ has a solution if and only if $P$ does. Let
  $\theta$ be a solution to $P = \tup{V, C}$. For every $e[\delta] \in
  C$, we have that $\theta|_{\iv(\delta)} \in \Pj_{\iv(\delta)}(P)$ by
  \cref{def:constraint-projection,def:intersection-vertices}, and the
  assignment $\mu$ that assigns the value $\theta(v)$ to each $v \in
  \displaystyle\bigcup_{e[\delta] \in C}\iv(\delta)$, and $*$ to every
  other variable is therefore a solution to $P_{CL}$.

  In the other direction, if $\theta$ is a solution to $P_{CL}$, then
  $\theta$ satisfies $\ic(e[\delta])$ for every $e[\delta] \in C$. By
  \cref{def:ind-const}, this means that $\theta|_{\iv(\delta)} \in
  \sol(\Pj_{\iv(\delta)}(P))$, and by
  \cref{def:constraint-projection}, there exists an assignment
  $\mu^{e[\delta]}$ with $\mu^{e[\delta]}|_{\iv(\delta)} =
  \theta|_{\iv(\delta)}$ that satisfies $e[\delta]$. By
  \cref{def:intersection-vertices}, the variables not in $\iv(\delta)$
  do not occur in any other constraint in $P$, so we can combine all
  the assignments $\mu^{e[\delta]}$ to form a solution $\mu$ to $P$
  such that for $e[\delta] \in C$ and $v \in \vars(\delta)$ we have
  $\mu(v) = \mu^{e[\delta]}(v)$.
\end{proof}

From \cref{thm:decomp-to-classic}, we get tractable and
fixed-parameter tractable classes of CSP instances with global
constraints.

\begin{corollary}
  \label{corollary:subproblem-csp}
  Let $\mathcal H$ be a class of hypergraphs, and $\languagesymbol$ a
  \languageword{} that allows partial assignment checking. If
  $\CSP(\mathcal H, \languagesymbol)$ has sparse intersections, then
  $\CSP(\mathcal H, \languagesymbol)$ is trac\-ta\-ble or in \FPT{} if
  $\CSP(\mathcal H, \Extlang)$ is.
\end{corollary}
\begin{proof}
  Let $\mathcal H$ and $\languagesymbol$ be given. By
  \cref{thm:decomp-to-classic}, we can reduce any $P \in \CSP(\mathcal
  H, \languagesymbol)$ to an instance $P_{CL} \in \CSP(\mathcal H,
  \Extlang)$ in polynomial time. Since $P_{CL}$ has a solution if and
  only if $P$ does, tractability or fixed-parameter tractability of
  $\CSP(\mathcal H, \Extlang)$ implies the same for $\CSP(\mathcal
  H,\languagesymbol)$.
\end{proof}

% An example of a class of instances $\CSP(\mathcal H, \Extlang)$ that
% is fixed-parameter trac\-ta\-ble is given by $\mathcal H$ having
% bounded submodular width (cf.~\cref{thm:submodular-width}).

\subsection{Applying \cref{corollary:subproblem-csp} to the CGP}
\label{sect:cgp-app}

Recall the connected graph partition problem (\cref{prob:cgp}): Given
a connected graph $G$, as well as natural numbers $\alpha$ and
$\beta$, can the vertices of $G$ be partitioned into bags of size at
most $\alpha$, such that no more than $\beta$ edges are broken. Using
the CSP encoding we gave in \cref{example:cgp-as-csp}, as well as
\cref{corollary:subproblem-csp}, we will show a new result, that this
problem is tractable if $\beta$ is fixed. To simplify the analysis, we
assume without loss of generality that $\alpha < |V|$, which means
that any solution has at least one broken edge.

We claim that if $\beta$ is fixed, then the constraint $C^\beta =
e^\beta[\delta^\beta]$ allows partial assignment checking, and has
only a polynomial number of satisfying assignments. The latter implies
that for any instance $P$ of the CGP,
$|\sol(\Pj_{\iv(\delta^\beta)}(P))|$ is polynomial in the size of $P$
for every subset of $\iv(\delta^\beta)$. Furthermore, we will show
that for the constraint $C^\alpha = e^\alpha[\delta^\alpha]$, we also
have that $|\sol(\Pj_{\iv(\delta^\alpha)}(P))|$ is polynomial in the
size of $P$. That $C^\alpha$ allows partial assignment checking
follows from a result by R\'egin~\cite{Regin96:aaai-generalized},
since the cardinality sets of $C^\alpha$ are intervals.

First, we show that the number of satisfying assignments to $C^\beta$
is limited. Since $C^\beta$ limits the number of ones in any solution
to $\beta$ or fewer, the number of satisfying assignments to this
constraint is the number of ways to choose up to $\beta$ variables to
be assigned one. This is bounded by
$\displaystyle\sum^{\beta}_{i=1}\binom{|E|}{i} \leq (|E|+1)^\beta$,
and so we can generate them all in polynomial time. 

Now, let $\theta$ be such a solution. How many solutions to $P$
contain $\theta$? Well, every constraint on $\{u, v, e\}$ with $e=1$
allows at most $|V|^2$ assignments, and there are at most $\beta$ such
constraints. So far we therefore have at most $(|E|+1)^\beta \times
|V|^{2\beta}$ assignments.

On the other hand, a ternary constraint with $e=0$ requires $u =
v$. Consider the graph $G_0$ containing for every constraint on $\{u,
v, e\}$ with $e=0$ the vertices $u$ and $v$ as well as the edge $\{u,
v\}$. Since the original graph was connected, every connected
component of $G_0$ contains at least one vertex which is in the scope
of some constraint with $e=1$. Therefore, since equality is
transitive, each connected component of $G_0$ allows at most one
assignment for each of the $(|E|+1)^\beta \times |V|^{2\beta}$
assignments to the other variables of $P$. We therefore get a total
bound of $(|E|+1)^\beta \times |V|^{2\beta}$ on the total number of
solutions to $P$, and hence to $\Pj_{\iv(\delta^\alpha)}(P)$.

The hypergraph of any CSP instance $P$ encoding the CGP has two
hyperedges covering the whole problem, so the hypertree width of this
hypergraph is two. Therefore, we may apply
\cref{corollary:subproblem-csp,thm:width-tract} to obtain tractability
when $\beta$ is fixed. 
As this problem is \NP-complete for fixed
$\alpha \geq 3$ \cite[p.~209]{Garey79:intractability}, $\beta$ is a
natural parameter to try and use.

% This tractability result relies on several properties of the
% problem. First, we exploited the fact that fixed $\beta$ guarantees
% few solutions to one of the EGC constraints, and from that we derived
% that the ternary constraints ensure that the whole CSP instance has
% few solutions.

As it happens, in this problem we can drop the requirement of partial
assignment checking for the constraint $C^\alpha$. All its variables
are intersection variables, and the instance has few solutions even if
we disregard $C^\alpha$. Thus, we need only check whether any of those
solutions satisfy $C^\alpha$, and checking whether an assignment to
the whole scope of a constraint satisfies it can always be done in
polynomial time by \cref{def:glob-const}. In the next section, we turn
this observation into a general result.

\section{Back Doors}
\label{sect:back-doors}

If a class of CSP instances includes constraints from a
\languageword{} that is not known to allow partial assignment
checking, we may still obtain tractability in some cases by applying
the notion of a back door set. A (strong) back door set
\cite{gaspers-backdoors,Williams03backdoorsto} is a set of variables
in a CSP instance that, when assigned, make the instance easy to
solve. Below, we are going to adapt this notion to individual
constraints.

\begin{definition}[Back door]
  \label{def:back-doors}
  Let $\languagesymbol$ be a global constraint \languageword{}. A
  \emph{back door} for a constraint $e[\delta] \in \languagesymbol$ is
  any set of variables $W \subseteq \vars(\delta)$ (called a back door
  set) such that we can decide in polynomial time whether a given
  assignment $\theta$ to a set of variables $\vars(\theta) \supseteq
  W$ is contained in an assignment that satisfies $e[\delta]$,
  i.e.~whether there exists $\mu \in e[\delta]$ such that
  $\mu|_{\vars(\theta)} = \theta$.
\end{definition}

Trivially, for every constraint $e[\delta]$ the set of variables
$\vars(\delta)$ is a back door set, since by \cref{def:glob-const} we
can always check in polynomial time if an assignment to
$\vars(\delta)$ satisfies the constraint $e[\delta]$.

The key point about back doors is that given a \languageword{}
$\Gamma$, adding to each $e[\delta] \in \Gamma$ with back door set $W$
an arbitrary set of assignments to $W$ produces a \languageword{}
$\Gamma'$ that allows partial assignment checking. Adding a set of
assignments $\Theta$ means to add $\Theta$ to the description, and
modify the algorithm $e$ to only accept an assignment if it contains a
member of $\Theta$ in addition to previous requirements. Furthermore,
given a CSP instance $P$ containing $e[\delta]$, as long as $\Theta
\supseteq \pi_{W}(\sol(P))$, adding $\Theta$ to $e[\delta]$ produces
an instance that has exactly the same solutions. This point leads to
the following definition.

\begin{definition}[Sparse back door cover]
  \label{def:sparse-back-door-cover}
  Let $\Gamma_{PAC}$ be a \languageword{} that allows partial
  assignment checking and $\Gamma_{BD}$ a \languageword{}. For every
  instance $P = \tup{V, C}$ over $\Gamma_{PAC} \cup \Gamma_{BD}$, let
  $P \cap \Gamma_{PAC}$ be the instance with constraint set $C' = C
  \cap \Gamma_{PAC}$ and set of variables $\bigcup \{V \cap
  \vars(\delta) \mid e[\delta] \in C'\}$.

  A class of CSP instances $\mathcal P$ over $\Gamma_{PAC} \cup
  \Gamma_{BD}$ has \emph{sparse back door cover} if there exists a
  constant $c$ such that for every instance $P = \tup{V, C} \in
  \mathcal P$ and constraint $e[\delta] \in C$, if $e[\delta] \not\in
  \Gamma_{PAC}$, then there exists a back door set $W$ for $e[\delta]$
  with $|\sol(\Pj_X(P \cap \Gamma_{PAC}))| \leq |P|^c$ for every $X
  \subseteq W$.
\end{definition}

Sparse back door cover means that for each constraint that is not from
a \languageword{} that allows partial assignment checking, we can in
polynomial time get a set of assignments $\Theta$ for its back door
set using \cref{alg:enum-solutions}, and so turn this constraint into
one that does allow partial assignment checking. This operation
preserves the solutions of the instance that contains this constraint.

\begin{theorem}
  \label{lemma:backdoors}
  If a class of CSP instance $\mathcal P$ has sparse back door cover,
  then we can in polynomial time reduce any instance $P \in \mathcal
  P$ to an instance $P'$ such that $\hyp(P) = \hyp(P')$ and $\sol(P) =
  \sol(P')$. Furthermore, the class of instances $\{P' \mid P \in
  \mathcal P\}$ is over a \languageword{} that allows partial
  assignment checking.
\end{theorem}
\begin{proof}
  Let $P =\tup{V, C} \in \mathcal P$. We construct $P'$ by adding to
  every $e[\delta] \in C$ such that $e[\delta] \not\in \Gamma_{PAC}$,
  with back door set $W$, the set of assignments $\sol(\Pj_W(P \cap
  \Gamma_{PAC}))$, which we can obtain using
  \cref{alg:enum-solutions}. By \cref{def:sparse-back-door-cover}, we
  have for every $X \subseteq W$ that $|\sol(\Pj_W(P \cap
  \Gamma_{PAC}))| \leq |P|^c$, so \cref{alg:enum-solutions} takes
  polynomial time since $\Gamma_{PAC}$ does allow partial assignment
  checking.

  It is clear that $\hyp(P') = \hyp(P)$, and since $\sol(\Pj_W(P \cap
  \Gamma_{PAC})) \supseteq \pi_W(\sol(P))$, the set of solutions stays
  the same, i.e.~$\sol(P') = \sol(P)$. Finally, since we have replaced
  each constraint $e[\delta]$ in $P$ that was not in $\Gamma_{PAC}$ by
  a constraint that does allow partial assignment checking, it follows
  that $P'$ is over a \languageword{} that allows partial assignment
  checking.
\end{proof}

One consequence of \cref{lemma:backdoors} is that we can sometimes
apply \cref{thm:decomp-to-classic} to a CSP instance that contains a
constraint for which checking if a partial assignment can be extended
to a satisfying one is hard. We can do so when the variables of that
constraint are covered by the variables of other constraints that do
allow partial assignment checking --- but only if the instance given
by those constraints has few solutions.

As a concrete example of this, consider again the encoding of the CGP
that we gave in \cref{example:cgp-as-csp}. The variables of constraint
$C^\alpha$ are entirely covered by the instance $P'$ obtained by
removing $C^\alpha$. As the entire set of variables of a constraint is
a back door set for it, and the instance $P'$ has few solutions
(cf.~\cref{sect:cgp-app}), this class of instances has sparse back
door cover. As such, the constraint $C^\alpha$ could, in fact, be
arbitrary without affecting the tractability of this problem. In
particular, the requirement that $C^\alpha$ allows partial assignment
checking can be dropped.

\section{Summary and Future Work}

In this paper, we have investigated properties that many structural
restrictions rely on to yield tractable classes of CSP instances with
explicitly represented constraints. In particular, we identify a
relationship between the number of solutions and the size of a CSP
instance as being one such property. Using this insight, we show that
known structural restrictions yield tractability for any class of CSP
instances with global constraints that satisfies this property. In
particular, the above implies that the structural restrictions we
consider yield tractability for classes of instances where every
global constraint has few satisfying assignments relative to its size.

To illustrate our result, we apply it to a known problem, the
connected graph partition problem, and use it to identify a new
tractable case of this problem. We also demonstrate how the concept of
back doors, subsets of variables that make a problem easy to solve
once assigned, can be used to relax the conditions of our result in
some cases.

As for future work, one obvious research direction to pursue is to
find a complete characterization of tractable classes of CSP instances
with sparse intersections. Another avenue of research would be to
apply the results in this paper to various kinds of valued CSP.

%\bibliographystyle{splncs03}
%\bibliography{mainrefs-thesis-copy}

\end{document}